\documentclass{article}

\usepackage{microtype}
\usepackage{graphicx}
\usepackage{subfigure}
\usepackage{booktabs} 
\usepackage{multirow}
\usepackage{enumitem}
\usepackage{hyperref}
\usepackage{thmtools,thm-restate}

\usepackage{stfloats}

\usepackage[accepted]{icml2025}

\usepackage{amsmath}
\usepackage{amssymb}
\usepackage{mathtools}
\usepackage{amsthm}

\usepackage[capitalize,noabbrev]{cleveref}
\usepackage{adjustbox}

\theoremstyle{plain}

\theoremstyle{definition}

\theoremstyle{remark}

\DeclareMathOperator*{\argmin}{arg\,min}

\usepackage[textsize=tiny]{todonotes}

\icmltitlerunning{Mitigating Label Errors via Rockafellian Relaxation}

\begin{document}

\twocolumn[
\icmltitle{Mitigating the Impact of Labeling Errors on Training via
Rockafellian Relaxation}



\icmlsetsymbol{equal}{*}

\begin{icmlauthorlist}
\icmlauthor{Louis L. Chen}{nps1}
\icmlauthor{Bobbie Chern}{comp}
\icmlauthor{Eric C. Eckstrand}{nps2}
\icmlauthor{Amogh Mahapatra}{}
\icmlauthor{Johannes O. Royset}{usc}
\end{icmlauthorlist}

\icmlaffiliation{nps1}{Department of Operations Research, Naval Postgraduate School, Monterey, CA}
\icmlaffiliation{nps2}{Data Science and Analytics Group, Naval Postgraduate School, Monterey, CA}
\icmlaffiliation{usc}{Department of Industrial and Systems Engineering, University of Southern California}
\icmlaffiliation{comp}{Meta Platforms, Inc., Location, Country}

\icmlcorrespondingauthor{Louis L. Chen}{louis.chen@nps.edu}
\icmlcorrespondingauthor{Eric Eckstrand}{eric.eckstrand@nps.edu}

\icmlkeywords{Machine Learning, ICML}

\vskip 0.3in
]



\printAffiliationsAndNotice{\icmlEqualContribution} 

\begin{abstract}
Labeling errors in datasets are common, arising in a variety of contexts, such as human labeling, noisy labeling, and weak labeling (i.e., image classification). Although neural networks (NNs) can tolerate modest amounts of these errors, their performance degrades substantially once error levels exceed a certain threshold. We propose a new loss reweighting, architecture-independent methodology, Rockafellian Relaxation Method (RRM) for neural network training. Experiments indicate RRM can enhance neural network methods to achieve robust performance across classification tasks in computer vision and natural language processing (sentiment analysis). We find that RRM can mitigate the effects of dataset contamination stemming from both (heavy) labeling error and/or adversarial perturbation, demonstrating effectiveness across a variety of data domains and machine learning tasks.
\end{abstract}

\section{Introduction}
\label{submission}
Labeling errors are systematic in practice, stemming from various sources. For example, the reliability of human-generated labels can be negatively impacted by incomplete information, or the subjectivity of the labeling task -- as is commonly seen in medical contexts, in which experts can often disagree on matters such as the location of electrocardiogram signal boundaries \cite{contaminationSurvey}, prostate tumor region delineation, and tumor grading \citep{prostateCancer}. As well, labeling systems, such as Mechanical Turk\footnote[2]{http://mturk.com} often find expert labelers being replaced with unreliable non-experts \citep{mturkNLP}. 
For all these reasons, it would be advisable for any practitioner to operate under the assumption that their dataset is contaminated with labeling errors, and possibly to a large degree.



In this paper, we propose a loss-reweighting methodology for the task of training a classifier on data having higher levels of labeling errors. More precisely, we provide a meta-algorithm that may ``wrap" any loss-minimization methodology to add robustness. The method relates to optimistic and robust  distributional optimization formulations aimed at addressing adversarial training (AT), which underscores our numerical experiments on NNs that suggest this method of training can provide test performance robust to high levels of labeling error, and to some extent, feature perturbation as well. Overall, we tackle the prevalent challenge of label contamination in training data, which is a critical obstacle for deploying robust machine learning models. The approach utilizes Rockafellian Relaxations \citep{royset2023} in addressing contaminated labels without the need for clean validation sets or sophisticated hyper-parameters - common constraints of current methodologies. This distinct capability represents a key contribution, making our approach more practical for handling large industrial datasets.

We proceed to discuss related works in Section \ref{relatedwork}, and our specific contributions to the literature. In Section \ref{methodology} we discuss our methodology in detail and provide some theoretical justifications that motivate the effectiveness of our methodology. Sections \ref{sec:data} and \ref{sec:architectures} discuss the datasets and NN model architectures upon which our experimental results are based. We then conclude with numerical experiments and results in section \ref{sec:results}.   

\section{Related Work}
\label{relatedwork}
Contaminated datasets are of concern, as they potentially pose severe threats to classification performance of numerous machine-learning approaches \citep{noiseImpact}, including, most notably, NNs \citep{noiseImpactNN1, noiseImpactNN3}. Naturally, there have been numerous efforts to mitigate this effect \citep{noisesurvey, contaminationSurvey}. These efforts can be categorized into robust architectures, robust regularization, robust loss function, \textit{loss adjustment}, and sample selection \citep{noisesurvey}. Robust architecture methods focus on developing custom NN layers and dedicated NN architectures. This differs from our approach, which is architecture agnostic and could potentially "wrap around" these methods. 
While robust regularization methods like data augmentation \citep{dataAug}, weight decay \citep{weightDecay}, dropout \citep{dropout}, and batch normalization \citep{batchNormalization} can help to bolster performance, they generally do so under lower levels of dataset contamination. Our approach, on the other hand, is capable of handling high levels of contamination, and can seamlessly incorporate methods such as these.
In label contamination settings, it has been shown that loss functions, such as robust mean absolute error (MAE) \citep{robustMAE}, early learning regularization (ELR) \citep{liu2020early}, and generalized cross entropy (GCE) \citep{robustGCE} are more robust than categorical cross entropy (CCE). Again, our method is not dependent on a particular loss function, and it is possible that arbitrary loss functions, including robust MAE and ELR, can be swapped into our methodology with ease. 
Our approach resembles the loss adjustment methods most closely, where the overall loss is adjusted based on a (re)weighting scheme applied to training examples.

In \textit{loss adjustment} methods, individual training example losses are typically adjusted multiple times throughout the training process prior to NN updates. These methods can be further grouped into loss correction, loss reweighting, label refurbishment, and meta-learning \citep{noisesurvey}. Our approach most closely resembles the \textit{loss reweighting} methods. Under this scheme each training example is assigned a unique weight, where smaller weights are assigned to examples that have likely been contaminated. This reduces the influence of contaminated examples. A training example can be completely removed if its corresponding weight becomes zero. 
Indeed, a number of loss reweighting methods are similar to our approach. For example, \citet{reweightICML} learn sample weights through the use of a noise-free validation set. \citet{activebias} assign sample weights based on prediction variances, and \citet{dualgraph} examine the structural relationship among labels to assign sample weights. However, we view the need by these methods for a clean dataset, or at least one with sufficient class balance, as a shortcoming, and our method, in contrast, makes no assumption on the availability of such a dataset. 

\citet{datacleansingneurips} propose a two-phased approach to noise cleaning. The first phase trains a standard neural network to determine the top-$m$ most influential training instances that influence the decision boundary; these are subsequently removed from the training set to create a cleaner dataset. In the second phase, the neural network is retrained using the cleansed training set. Their method demonstrates superior validation accuracy for various values of $m$ on MNIST and CIFAR-10. Although impressive, their method does not address the fact that most industrial datasets have a reasonably large amount of label contamination \citep{noisesurvey} which, upon complete cleansing, could also remove informative examples that lie close to the decision boundary. Additionally,  the value of $m$ is an additional hyper-parameter that could require significant tuning on different datasets and sources.

Perhaps the most similar approach to this work is that of \citet{reweightICML}, in which label noise and class imbalance are treated simultaneously via the learning of exemplar weights. 
The biggest drawback of their method is that it requires a clean validation set, which in practice is almost impossible to obtain; if it were possible, it would not be very prohibitive to clean the entire dataset. Noise is typically an artifact of the generative distribution, which cannot be cherry-picked as easily in practice. Our approach does not require a clean dataset to be operational.  Furthermore, even though we do not explicitly treat class imbalance, we are not affected by class imbalance, as demonstrated in our experiments with an open-source Toxic Comments dataset (Appendix Section \ref{sec:toxicresults}).

\section{Methodology}
\label{methodology}

\subsection{Mislabeling}\label{sec:ncar}
Let $\mathcal{X}$ denote a \textit{feature} space, with $\mathcal{Y}$ a corresponding \textit{label} space. Then $\mathcal{Z}:= \mathcal{X} \times \mathcal{Y}$ will be a collection of feature-label pairs, with an unknown probability distribution $D$. Throughout the forthcoming discussions, $\{(x_i, y_i)\}_{i=1}^N$ will denote a sample of $N$ feature-label pairs, for which some pairs will be mislabeled. More precisely, we begin with a collection $(x_i, \tilde{y}_i)$ drawn i.i.d. from $D$, but there is some unknown set $C \subsetneq \{1, \ldots, N\}$ denoting (contaminated) indices for which $y_i = \tilde{y}_i$ if and only if $i \notin C.$ For indices $i \in C,$ $y_i$ is some incorrect label, say selected uniformly at random, following the Noise Completely at Random (NCAR) model \citep{contaminationSurvey} also known as \textit{uniform label noise}; however, the particular choice of noise/contamination matters not. While the experiments of Section \ref{sec:results} implement NCAR, we also consider non-uniform contamination in Appendix Section \ref{sec:nonuniform}.

\subsection{Rockafellian Relaxation Method (RRM)}\label{sec:rockafellianrelaxation}
We adopt the empirical risk minimization (ERM) \citep{vapnikERM} problem formulation:
\begin{equation}
\min_{\theta} \frac{1}{N}\sum_{i = 1}^N  J(\theta; x_i,y_i) + r(\theta)
\label{eq:empirical_risk_minimizaion}
\end{equation}
as a baseline against which our method is measured. Given an NN architecture with (learned) parameter setting $\theta$ that takes as input any feature $x$ and  outputs a prediction $\hat{y}$, $J(\theta; x,y)$ is the loss with which we evaluate the prediction $\hat{y}$ with respect to $y$. Finally, $r(\theta)$ denotes a regularization term. 

In ERM it is common practice to assign each training observation $i$ a probability $p_i = 1/N$. However, when given a contaminated dataset, we may desire to remove those samples that are affected; in other words, if $C \subsetneq \{1, ... , N\}$ is the set of contaminated training observations, then we would desire to set the  probabilities in the following alternative way: 
\vspace{-2.5mm}
\begin{equation} \label{eqn:: DesiredContaminatedProbs}
p = (p_i)_{i \in [N]} \text{ with } p_{i} = 
\begin{cases}
    0,& \text{if } i \in C\\
    \frac{1}{N - |C|},& \text{if } i \in [N] \setminus C,
\end{cases}
\end{equation}
\vspace{-1mm}
where $|C|$ is the cardinality of the unknown set $C$. In this work, we provide a procedure - the \textit{Rockafellian Relaxation Method} (RRM) - with the intention of aligning the $p_i$ values closer to the desired (but unknown) $p$ of \eqref{eqn:: DesiredContaminatedProbs} in self-guided, automated fashion. It does so by adopting the Rockafellian Relaxation approach of \citep{royset2023}. More precisely, with $\Delta(N) \subseteq \mathbb{R}^N$ denoting the standard probability simplex, and $p^N$ denoting the equally weighted (empirical) distribution $\{p^N_i\}_{i \in [N]} = \{1/N\}$, we consider the problem
\begin{align} 
    &\min_\theta \min_{p \in \Delta(N) }\mathbb{E}_{(x,y) \sim p } \left[J(\theta; x,y)\right] + \gamma \cdot d_{TV}(p^N, p)\nonumber \\
    &= \min_\theta \Bigl[ \underbrace{\min_{u \in U} \sum_{i = 1}^N  (\frac{1}{N} + u_i) \cdot J(\theta; x_i,y_i) + \frac{\gamma}{2} \lVert u \rVert_{1}}_{v(\theta)}\Bigr] , \label{eqn:: RRM_Problem}
\end{align}
where $U := \{u \in \mathbb{R}^N: \sum_{i=1}^N u_i = 0, \frac{1}{N} + u_i \geq 0 \;\; \forall i = 1, \ldots, N\}$, $\gamma > 0$, and $d_{TV}(p^N, p):= \frac{1}{2}\sum_i |p^N_i - p_i|$.
In words, in \eqref{eqn:: RRM_Problem} an alternative distribution $p$ may now replace the empirical, at the cost of $\gamma$ per unit of \textit{total variation}. Indeed, for any $u \in U$, we obtain a probability distribution $p^u \in \mathbb{R}^N$, defined via $p^u_i = 1/N + u_i,$ so that $\|u\|_1$ is the total variation distance between the empirical distribution and $p^u.$

We proceed to comment on this problem that is nonconvex in general, before providing an algorithm.

\subsection{Analysis and Interpretation of Rockafellian Relaxation} \label{sec: interpretRockafellian}
Although problem (\ref{eqn:: RRM_Problem}) is nonconvex in general, the computation of $v(\theta)$ for any fixed $\theta$ amounts to a linear program. The following result characterizes the complete set of solutions to this linear program, and in doing so, provides an interpretation of the role that $\gamma$ plays in the loss-reweighting action of RRM.

\begin{restatable}{theorem} {SolvingForU}\label{theorem:: SovlingForU}
Let $\gamma > 0$ and $c = (c_1, \dots, c_N) \in \mathbb{R}^N$, with $c_{min}:= \min_{i} c_i$, and $c_{max}:= \max_i c_i$. Write $I_{min} := \{i: c_i = c_{min}\}$, $I_{mid}:= \{i: c_i \in (c_{min}, c_{min} + \gamma)\}$, $I_{big} := \{i: c_i = c_{min} + \gamma\}$, and for any $S_1 \subseteq I_{min},$ $S_2 \subseteq I_{big}$, define the polytope\\
     $U^*_{S_1, S_2}:= 
     \begin{Bmatrix} 
     & u^* \in U\\
      & u^*_i \geq 0 \;\; \forall i \in I_{min}\\
      u^* : & u_i^* = 0 \;\; \forall i \in S_1 \cup I_{mid}\\
      & u^*_i \leq 0 \;\; \forall i \in I_{big}\\
      &u_i^* = -\frac{1}{N} \;\; \forall i 
     \in S_2 \\
       &u_i^* = -\frac{1}{N} \;\; \forall i: c_i > c_{min} + \gamma\\
 \end{Bmatrix}.
 $\\
Then  $\mathrm{conv}\left(\bigcup_{S_1, S_2} U^*_{S_1, S_2}\right)$ is equivalently
\begin{equation} \label{eq:: inner - U} 
   \underset{u \in U}{\argmin} \sum_{i = 1}^N  (\frac{1}{N} + u_i) \cdot c_i + \frac{\gamma}{2} \lVert u \rVert_{1}.
\end{equation} 
\end{restatable}

The theorem explains that the construction of any optimal solution $u^*$ essentially reduces to categorizing each of the losses among $\{c_i = J(\theta; x_i, y_i)\}_{i=1}^N$ as ``small" or ``big", according to their position in the partitioning of $[c_{min}, \infty) = [c_{min}, c_{min} + \gamma) \cup [c_{min} + \gamma, \infty)$. For losses that occur at the break points of $c_{min}$ and $c_{min} + \gamma$, this classification can be arbitrary - hence, the use of $S_1$ and $S_2$ set configurations to capture this degree of freedom.

In particular, those points with losses $c_i$ exceeding $c_{min} + \gamma$ are down-weighted to zero and effectively removed from the dataset. And in the event that $c_{max} - c_{min} < \gamma$, no loss reweighting occurs. In this manner, while lasso produces sparse solutions in the model parameter space, RRM produces sparse weight vectors by assigning zero weight to data points with high losses. 

It is worth mentioning that upon setting $S_1 = S_2 = \emptyset$ in Theorem \ref{theorem:: SovlingForU}, a particular solution for \eqref{eq:: inner - U} can be derived.

\begin{restatable}{corollary} {QuickSolveU}\label{corollary:: QuickSolveU}
    Let $\theta^* \in \argmin_{\theta} v(\theta)$
    and 
    $\chi = \{i: J(\theta^*; x_i, y_i)) > \min_j J(\theta^*; x_j, y_j))+ \gamma\}$.
     Let $u^*$ be given by $\{u^*_i\}_{i \in {I_{min}}} = \{\frac{|\chi|}{N\cdot |I_{min}|} \}$, $\{u_i^*\}_{i \in I_{mid} \cup I_{big}} =  \{0\}$, and $u_i^* = \frac{-1}{N}$ for all $i$ such that $c_i > c_{min} + \gamma$. Then $(\theta^*, u^*)$ solves \eqref{eqn:: RRM_Problem}; further, $d_{TV}(p^N, p^*) = \frac{|\chi|}{N}.$
\end{restatable}
In words, transferring $\frac{|\chi|}{N}$ probability mass away from the group of highest-cost examples and distributing it uniformly among the examples comprising $I_{min}$ is optimal. 

If $\chi$ converges over the course of any algorithmic scheme, e.g., Algorithm \ref{alg: RRM}, to some set $C$, then we can conclude that these data points are effectively removed from the dataset even if the training of $\theta$ might proceed. This convergence was observed in the experiments of Section \ref{sec:results} (see Table \ref{tab:mnistadv_uvalues}). 
It is hence of possible consideration to tune $\gamma$ for consistency with an estimate $C' \in [0,1]$ of labeling error in the dataset $\{(x_i, y_i)\}_{i=1}^N$. More precisely, we may tune $\gamma$ so that $\frac{|\chi|}{N} \approx C'.$ 

\subsection{RRM and Optimistic Wasserstein Distributionally Robust Optimization} \label{sec:: RRM-Wasserstein}

In this section, we discuss RRM's relation to distributionally robust and optimistic optimization formulations. Indeed, $(\ref{eqn:: RRM_Problem})$'s formulation as a min-min problem bears resemblance to optimistic formulations of recent works, e.g., \cite{optimisticDistributionally}. We will see as well that the minimization in $u$, as considered in Theorem \ref{theorem:: SovlingForU}, relates to an approximation of a data-driven Wasserstein Distributionally Robust Optimization (DRO) formulation \citep{distributionally2}. 

\subsubsection{Loss-reweighting via Data-Driven Wasserstein Formulation}
For this discussion, as it relates to reweighting, we will lift the feature-label space $\mathcal{Z} = \mathcal{X} \times \mathcal{Y}$. More precisely, we let $\mathcal{W}:= \mathbb{R}_+$ denote a space of \textit{weights}. Next, we say $\mathcal{W} \times \mathcal{Z}$ has an unknown probability distribution $\mathcal{D}$ such that $\pi_\mathcal{Z} \mathcal{D}= D$ and $\Pi_\mathcal{W}\mathcal{D}(\{1\}) = 1$. In words, all possible (w.r.t. $D$) feature-label pairs have a weight of $1.$ Finally, we define an \textit{auxiliary loss} 
$\ell: \mathcal{W} \times \mathcal{Z} \times \Theta$ by $
\ell(w, z; \theta):= w\cdot J(x,y; \theta)$, for any $z = (x,y) \in \mathcal{Z}$. 

Given a sample $\{(1, x_i,y_i)\}_{i=1}^N$, just as in Section \ref{sec:rockafellianrelaxation}, we can opt not to take as granted the resulting empirical distribution $\mathcal{D}_N$ because of the possibility that $|C|$-many have incorrect labels (i.e., $y_i \neq \tilde{y}_i)$. Instead, we will admit alternative distributions obtained by shifting the $\mathcal{D}_N$'s probability mass off ``contaminated" tuples $(1, x_i,y_i)_{i \in C}$ to possibly $(0, x_i, y_i)$, $(1, x_i, \tilde{y}_i)$, or even some other tuple $(1, x_j, \tilde{y}_j)$ with $j \notin C$ for example - equivalently, eliminating, correcting, or replacing the sample, respectively. In order to admit such favorable corrections to $\mathcal{D}_N,$ we can consider 
the optimistic \citep{optimisticDistributionally, distributionally2} data-driven problem 
\begin{equation}
    \min_\theta \left(v_{N}(\theta):= \min_{\tilde{\mathcal{D}}: W_1(\mathcal{D}_N, \tilde{\mathcal{D}})\leq \epsilon} \mathbb{E}_{\tilde{\mathcal{D}}} \left[\ell(w, z; \theta)
    \right] 
    \right),
\end{equation}
in which for each parameter tuning $\theta$, $v_N(\theta)$ measures the expected auxiliary loss with respect to the most favorable distribution within an $\epsilon$-prescribed $W_1$ (1-Wasserstein) distance of $\mathcal{D}_N$. It turns out that a budgeted deviation of the weights alone (and not the feature-label pairs) can approximate (up to an error diminishing in $N$) $v_N(\theta)$. More precisely, we derive the following approximation along similar lines to \citet{distributionally2}.
\begin{restatable}{prop} 
{Optimistic}\label{prop:: Optimistic}
Let $\epsilon > 0$, and suppose for any $\theta$, $\max_{(x,y) \in \mathcal{Z}}|J(\theta; x, y)| < \infty.$ 
Then there exists $\kappa \geq 0$ such that for any $\theta$, the following problem
\begin{align*}
v_N^{MIX}(\theta):= \min_{u_1, \ldots, u_N} ~ & \sum_{i = 1}^N  (\frac{1}{N} + u_i) \cdot J(\theta; x_i,y_i) + \gamma_\theta 
\lvert\lvert u \rvert\rvert_1
\\
& \text{s.t.}  ~~u_i + \frac{1}{N} \geq 0 \;\; i = 1, \ldots, N
\end{align*}
satisfies 
$v_N(\theta) + \frac{\kappa}{N} \geq v_N^{MIX}(\theta) \geq v_N(\theta).$

In particular, $-\gamma_\theta \leq \min_i J(\theta; x_i, y_i)$, and any member $i \in \{i: J(\theta; x_i, y_i) > \gamma_\theta \}$ is down-weighted to zero, i.e., $u_i^* = -\frac{1}{N}$ for any $u^*$ solving $v_N^{MIX}(\theta).$
\end{restatable}

In summary, while the optimistic Wasserstein formulation would permit correction to $\mathcal{D}_N$ with a combination of reweighting and/or feature-label revision, the above indicates that a process focused on reweighting alone could accomplish a reasonable approximation; further, upon comparison to (\ref{eqn:: RRM_Problem}), we see that RRM is a constrained version of this approximating problem, that is,  
\[v(\theta) \geq v_N^{MIX}(\theta) \geq v_{N}(\theta).\] 
Hence, in some sense, we can say that RRM is an optimistic methodology but that it is less optimistic than the data-driven Wasserstein approach.


\subsection{A-RRM/RRM Algorithm }
 Towards solving problem (\ref{eqn:: RRM_Problem}) in the two decisions $\theta$ and $u$, we proceed iteratively with a block-coordinate descent heuristic outlined in Algorithm \ref{alg: RRM}, whereby we update the two separately in cyclical fashion. In other words, we update $\theta$ while holding $u$ fixed, and we update $u$ whilst holding $\theta$ fixed. The update of $\theta$ is an SGD step on a batch of $s$-many samples. The update of $u$ is a linear program, but can also be completed quicker via Corollary \ref{corollary:: QuickSolveU}. 
\begin{algorithm}[h!]
\caption{(Adversarial) Rockafellian Relaxation Algorithm (A-RRM/RRM)}
\begin{algorithmic}  \label{alg: RRM}
\REQUIRE Loss Function $J(\theta; x, y)$, training perturbation $\epsilon \in [0,1]$, Number of epochs $\sigma,$ Batch size $s \geq 1$, learning rate $\eta > 0,$ threshold $\gamma > 0$, reweighting step $\mu \in (0,1)$. (optional) contamination estimate $C' \in [0,1]$.

\STATE $u \gets 0 \in \mathbb{R}^N$
\STATE $Iter \gets 0$
\REPEAT 
\STATE{$Iter \gets Iter + 1$}
\STATE $\theta \gets \texttt{GradientStep}(\sigma, s, \eta, \epsilon, \theta, u)$
\IF{$C'$ provided}
\STATE $\ell_{1-C'} \gets \max \{\ell:  \frac{|\{j: J(\theta; x_j^b, y_j^b) > \ell\}|}{N} \geq C' \}$
\STATE $\gamma \gets \ell_{1-C'} - \min_i J(\theta; x_i^b, y_i^b)$
\STATE $\mu \gets 1$
\ENDIF
\STATE $u \gets $ \texttt{Re-weight}($\gamma, \mu$)
\UNTIL{Desired Validation Accuracy or Loss, or $Iter$ is sufficiently large}
\STATE Return $\theta$
\end{algorithmic}
\end{algorithm}

\begin{algorithm}[h!] 
\caption{\texttt{GradientStep}($\sigma, s, \eta, \epsilon, \theta, u$)}
\begin{algorithmic} \label{alg::GradientStep}
\REQUIRE $\sigma \in \mathbb{Z}_{\geq 1}, s \in \mathbb{Z}_{\geq 1}, u \in U, \eta > 0, \epsilon \in [0,1], \theta$ 
    \FOR{$e = 1, \ldots, \sigma$}
    \FOR{$b = 1, \ldots, \lceil\frac{N}{s}\rceil$}
        \STATE $\{(x^b_i, y^b_i)\}_{i=1}^s \gets \text{draw batch from } \{(x_i, y_i)\}_{i=1}^N $
        \FOR{i = 1, \ldots, s}
            \STATE $x_i^b \gets x_i^b + \epsilon \cdot sign\left(\nabla_{x} J(\theta; (x^b_i, y^b_i))\right)$
        \ENDFOR
        \STATE $\theta \gets \theta - \eta \sum_{i = 1}^s \left(\frac{1}{N} + u_i\right) \cdot \nabla_{\theta} J(\theta; (x^b_i, y^b_i))$
        
    \ENDFOR
\ENDFOR
\STATE Return $\theta$
\end{algorithmic}
\end{algorithm}

\begin{algorithm}[h!] 
\caption{\texttt{Re-weight}($\gamma, \mu$)}
\begin{algorithmic}  \label{alg: subRoutine}
\REQUIRE threshold $\gamma > 0$, re-weighting step $\mu \in (0,1)$. 
\STATE $u^* \gets \min_{u \in U} \sum_{i = 1}^N  \left(\frac{1}{N} + u_i\right) \cdot J(\theta; x_i,y_i) + \gamma \lVert u \rVert_{1}$
\STATE Return $u \gets \mu u^{*} + (1-\mu)$
\end{algorithmic}
\end{algorithm}

Apart from the learning rate $\eta$ that is (industry) standard in gradient-based algorithms, Algorithm \ref{alg: RRM} can be parameter-less; more precisely, the step-size $\mu$ and threshold $\gamma$ parameters can in fact be auto-tuned, in a manner that follows from Corollary \ref{corollary:: QuickSolveU},  so as to precisely control how many examples are pruned in each re-weighting step, as might be desired or guided by a contamination estimate $C'.$  

\subsubsection{Auto-Tuning: Precise Sample Pruning with a Contamination Estimate $C'$}\label{sec:autotuning}

As outlined in Algorithm \ref{alg: RRM}, if an estimate $C'$ of the true contamination $C$ is on hand, then upon setting $\gamma = \ell_{1-C'} - \min_i J(\theta; x_i, y_i)$, as detailed in Algorithm \ref{alg: RRM}, where 
$\ell_{1-C'}$ is approximately the $(1-C')-$th quantile of the costs, or equivalently, $\max \{J(\theta; x_i,y_i):  \frac{\{j: J(\theta; x_j, y_j) > J(\theta; x_i,y_i)\}}{N} \geq C' \}$, then the fraction of observations ``pruned" $\frac{|\chi|}{N} = \frac{|\{i: J(\theta; x_i, y_i) > \gamma + \min_j J(\theta; x_j, y_j)\}|}{N}$ is at least $C'$, once the hyperparameter $\mu$ is also fixed to $1$. This is essentially the re-weighting action stated in Corollary \ref{corollary:: QuickSolveU}. We refer the reader to Section \ref{sec: ExperimentControllingChi} for experiments on this form of RRM.




\subsubsection{ A-RRM $(\epsilon > 0)$ versus RRM $(\epsilon = 0)$}
We will refer to Algorithm \ref{alg: RRM} as A-RRM versus RRM depending on $\epsilon$. The key difference lying in the execution of (Algorithm \ref{alg::GradientStep}'s) \texttt{GradientStep}.
\subsubsection*{A-RRM: (Adversarial) Feature Perturbation ($\epsilon > 0$) \& Label Contamination $C$}
When $\epsilon > 0$, then \texttt{GradientStep} executes a Fast Gradient Sign Method (FGSM) \citep{fgsm} adversarial attack, whereby each time a batch $(x^b_i, y^b_i)_{i = 1}^s$ is drawn, every $x^b_i$ is perturbed via $x^b_i + \epsilon \cdot sign(\nabla_x J(\theta,x^b_i,y^b_i))$. We note that other adversarial perturbation methods can be substituted in place of this choice of attack. Then a Stochastic gradient descent (SGD) step is taken. 
\subsubsection*{RRM: Label Contamination $C$ Only ($\epsilon = 0$)}
When $\epsilon = 0,$ then   \texttt{GradientStep} simply reduces to an SGD step.

\subsubsection{On Complexity} 
Each iteration of Algorithm \ref{alg: RRM} is comprised of two different tasks: 1.) \texttt{GradientStep} and 2.) \texttt{Re-weight}.  
Task 1's gradient step calculations are entirely standard practice. Task 2 amounts to the polynomial-sized (in training data) linear program of \eqref{eq:: inner - U}, for which theoretical efficiency is well-established, and commercial solvers abound (we used CPLEX 12.8.0); hence, complexity is poly-time in the training data size, and the scaling of computation with larger datasets should not be limiting in nature. In fact, Theorem 
\ref{theorem:: SovlingForU} and Corollary \ref{corollary:: QuickSolveU} indicate there exists an exploitable structure for $N\mathrm{log}N$ execution of Task 2.
We refer the reader to Sections 6.1 and 6.2 for experiments on MNIST-3 and CIFAR-10 with average computation times of Algorithm \ref{alg: subRoutine} reported. 

\section{Datasets}\label{sec:data}
The datasets utilized to evaluate RRM possess high-quality labels. So, to accurately assess our method under contaminated training regimes, we perturb the training datasets to achieve various types and levels of contamination $C$. In all cases the test set used for evaluation is nearly \citep{northcutt2021pervasive} pristine. A detailed description of the datasets utilized is provided below.

\textbf{MNIST-10(3)} \citep{lecun-mnisthandwrittendigit-2010}: A multi-class classification dataset consisting of 70000 images of digits zero through nine. 60000 digits are set aside for training and 10000 for testing. Training labels are swapped for different, randomly selected digits to achieve desired levels of $C$. In addition to the 10-class dataset, we utilize a smaller 3-class version consisting only of digits 0, 1, and 2.

\textbf{CIFAR-10} \citep{alex2009learning} A multi-class classification dataset consisting of 60000 32x32 color images of 10 classes, with 6000 images per class. 50000 images are set aside for training and 10000 for testing. As with MNIST, to achieve the desired $C$ training labels are swapped for different, randomly selected image classes.

\section{Architectures}\label{sec:architectures}

\begin{table*}[h!]
    \centering
    \caption{{\bf \underline{CIFAR-10}} Test accuracy (\%) comparison between CCE, MAE, and MSE loss functions against baseline GCE method across contamination levels. RRM-wrapped accuracy results are in parentheses. Values marked with ($\dagger$) are from \cite{robustGCE}.}
    \begin{tabular}{c||c|c|c}
        \multirow{2}{*}{Method} & \multicolumn{3}{c}{Contamination $C$} \\
        \cline{2-4}
        & 10\% & 20\% & 30\%    \\
        \hline
        \hline
        \multicolumn{1}{c||}{GCE} & 91.53 $\pm$ 0.25 & 89.70 $\pm$ $0.11^{\dagger}$ & 89.64 $\pm$ 0.17 
        \\
        \hline
        \\
        \hline
        \multicolumn{1}{c||}{CCE (+RRM)} & 89.94 $\pm$ 0.29 (92.20 $\pm$ 0.26) & 86.98 $\pm$ $0.44^{\dagger}$ (90.44 $\pm$ 0.31) & 81.90 $\pm$ 0.86 (88.49 $\pm$ 0.33) 
        \\    
        \hline
        \multicolumn{1}{c||}{MAE (+RRM)} & 92.04 $\pm$ 0.25 (90.10 $\pm$ 3.56) & 83.72 $\pm$ $3.84^{\dagger}$ (87.82 $\pm$ 4.07) & 88.28 $\pm$ 3.52 (82.98 $\pm$ 6.55) 
        \\
        \hline
        \multicolumn{1}{c||}{MSE (+RRM)} & 91.83 $\pm$ 0.25 ({\bf 93.04} $\pm$ 0.17) & 89.43 $\pm$ 0.37 ({\bf 91.43} $\pm$ 0.32) & 86.34 $\pm$ 0.22 ({\bf 89.95} $\pm$ 0.55) 
        \\
        \hline
    \end{tabular}
    \label{tab:rrm_v_gce}
\end{table*}

We do not strive to develop a novel NN architectures capable of defeating current state-of-the-art (SOA) performance in each data domain. Nor do we focus on developing \textit{robust architectures} as described in \citep{noisesurvey}. Rather, we select a reasonable NN architecture and measure model performance with and without the application of RRM. This approach enables us to demonstrate the general superiority of RRM under varied data domains and NN architectures. We discuss the underlying NN architectures that we employ in this section. 

\textbf{MNIST-10}: The MNIST dataset has been studied extensively and harnessed to investigate novel machine-learning methods, including CNNs \citep{mnist_survey}. For experiments utilizing MNIST-10, we adopt a basic CNN architecture with a few convolutional layers. The first layer has a depth of 32, and the next two layers have a depth of 64. Each convolutional layer employs a kernel of size three and the ReLU activation function followed by a max-pooling layer employing a kernal of size 2. The last convolutional layer is connected to a classification head consisting of a 100-unit dense layer with ReLU activation, followed by a 10-unit dense layer with softmax activation.  In total, there are 159254 trainable parameters. Categorical cross-entropy is employed for the loss function. 

\textbf{MNIST-3}: For experiments utilizing MNIST-3 a simple fully-connected architecture is instead utilized that consists of a few layers. The first two layer consists of 320 units, and the third layer consists of 200 units. Each of these dense layers employs the ReLU activation function. The output layer consists of a 3-unit dense layer with softmax activation. In total, there are 417880 trainable parameters. 

\textbf{CIFAR-10}: For comparison purposes, the CIFAR-10 experiments adopt the ResNet-34 architecture utilized in \citet{robustGCE}, along with their data preprocessing and augmentation scheme. We perform 32x32 random crops following padding with 4 pixels on each side, horizontal random flips, and per-pixel mean subtraction.

\section{Experiments and Results} \label{sec:results} 
We perform experiments to see how Algorithm \ref{alg: RRM} performs under two kinds of regimes: 
(a) Label Contamination only; (b) Feature Perturbation and Label Contamination. 

For regime (a), we evaluated RRM on MNIST-3 and CIFAR-10; we refer the reader to the Appendix (Section \ref{sec:AdditionalExperiments}) for further experimentation on Toxic Comments, IMDb and Tissue Necrosis. 
For regime (b), we evaluated A-RRM on MNIST-10.  

\subsection{RRM with a Provided Contamination Estimate $C'$ of $C$} \label{sec: ExperimentControllingChi}
In this section, we suppose access to a lower bound $C'$ on the contamination level $C,$ and leverage this estimate with the auto-tuning of $\mu$ and $\gamma$ as discussed in Section \ref{sec:autotuning}.

\subsubsection{Case of Perfect Estimate $C' = C$}\label{sec:cperfect}
{\bf Implementation Details:}
Using Tensorflow 2.10 \citep{tensorflow2015-whitepaper}, 180 epochs of training are performed using the CIFAR-10 dataset and architecture of \cref{sec:data,sec:architectures} for CCE, MAE and MSE loss functions, and training label contamination $C$ of 10\%, 20\%, and 30\%. For comparison, RRM is executed for the same loss functions and with a provided $C'$ equal to $C$ for 12 iterations with $\sigma = 15$ epochs per iteration, also for a total of 180 epochs. Stochastic gradient descent (SGD) with a learning rate ($\eta$) of $10.0$, momentum of $0.9$ with Nesterov momentum enabled, and a weight decay of $10^{-5}$ are employed for both sets of experiments.

The results from these experiments are also compared to those achieved by GCE. For $C=20$\% results are available directly from \citet{robustGCE}. For $C=10$\% and $C=30$\% results are re-computed using a Pytorch \citep{pytorch2024} implementation \cite{gce_impl} of GCE under the authors' reported hyperparameter settings. 

Each experiment is executed 5 times and the average test set accuracy and standard deviation is reported in Table \ref{tab:rrm_v_gce}, with the RRM results placed in parentheses. 

\textbf{Analysis:} We benchmark the benefit of RRM-wrapping against GCE, an NN-training method robust to noisy labeling. As Table \ref{tab:rrm_v_gce} indicates, both CCE and MSE methods benefit from an RRM approach, across contamination levels. The benefit of RRM for MAE is less definitive. However, the RRM approach with MSE achieves the best results, even when compared to GCE's performance. The additional computation time incurred from the \texttt{Re-weight}($\gamma, \mu$) step is an average of 3.88 seconds per execution of every iteration, for a total of 46.56 seconds.

\subsubsection{Case of Conservative Estimate $C' > C$}
We repeated the experimental setup of Section \ref{sec:cperfect} to evaluate the benefit of RRM for MSE, but now with an estimate $C'>C$ provided to RRM. 

\textbf{Analysis:} Upon comparing Table \ref{tab:rrm_v_gce_conservative} 
with Table \ref{tab:rrm_v_gce}, we confirm that as the estimate $C'$ is made larger and larger than $C,$ conferred benefit by RRM is reduced and is generally not as high as when $C' = C$. However, despite this, RRM-wrapping still generally outperforms MSE alone.

\begin{table}[h!]
    \centering
    \caption{ {\bf \underline{CIFAR-10}} Test accuracy (\%) of RRM-wrapped MSE with Conservative Estimate $C'$ of $C$ on classification task .}
    \begin{adjustbox}{width=\columnwidth,center}
    \begin{tabular}{c|c|c|c|c|}
        \multicolumn{2}{c}{} & \multicolumn{3}{c}{Contamination $C$} \\
        \cline{3-5}
        \multicolumn{2}{c|}{} & 10\% & 20\% & 30\% \\ 
        \hline
        \multirow{2}{*}{$C'$} 
        & $C + 10\%$ & 91.73 $\pm$ 0.28 & 90.60 $\pm$ 0.40 & 88.47 $\pm$ 0.73 \\ 
        \cline{2-5}
         & $C + 5\%$ & 92.55 $\pm$ 0.15 & 91.29 $\pm$ 0.41 & 89.43 $\pm$ 0.64 \\ 
        \hline
    \end{tabular}
    \label{tab:rrm_v_gce_conservative}
    \end{adjustbox}
\end{table}

\subsection{Enhancing Loss Function Methodologies with RRM}
RRM can enhance the performance of common loss functions that do not account for label contamination, including categorical cross-entropy (CCE), mean absolute error (MAE), and mean squared error (MSE). In fact, the following experiments indicate RRM can even enhance those methods devised with robustness to noisy labels, such as early-learning regularization (ELR).

\begin{table}[ht!]
    \centering
    \caption{{\bf \underline{MNIST-3}} Test accuracy (\%) comparison of CCE, MAE, and MSE against ELR across contamination levels. RRM-wrapped accuracy results are in parentheses.}
    \begin{tabular}{c||c|c|c}
        \multirow{2}{*}{Method} & \multicolumn{3}{c}{Contamination Level} \\
        \cline{2-4}
        & 55\% & 60\% & 65\% \\
        \hline
        \hline
        \multicolumn{1}{c||}{ELR (+ RRM)} & 98 (99) & 97 (98) & 82 (87) \\
        \hline
        \\
        \hline
        \multicolumn{1}{c||}{CCE (+ RRM)} & 90 (98) & 77 (96) & 46 (67) \\
        \hline
        \multicolumn{1}{c||}{MAE (+ RRM)} & 98 (98) & 96 (98) & 62 (68) \\
        \hline
        \multicolumn{1}{c||}{MSE (+ RRM)} & 90 (98) & 74 (97) & 45 (87) \\    
        \hline
    \end{tabular}
    \label{tab:rrm-wrapping}
\end{table}

\begin{table*}[bp]
    \small
    \begin{center}
        \caption{{\bf \underline{MNIST-10}} Test accuracy (\%) for AT and A-RRM under different levels of contamination $C$, training perturbation $\epsilon = 1.0$, and test-set adversarial perturbation $\epsilon_{test}$. }
        \label{tab:mnistadv}
        \begin{tabular}{c|c||c|c|c|c|c|c|c|c|c|c}
            \multicolumn{2}{c}{} & \multicolumn{10}{c}{$C$} \\ \hline
            &  & \multicolumn{2}{c}{0\%} \vline & \multicolumn{2}{c}{5\%} \vline & \multicolumn{2}{c}{10\%} \vline & \multicolumn{2}{c}{20\%} \vline & \multicolumn{2}{c}{30\%} \\ \cline{3-12}
            &  & \multicolumn{1}{c|}{AT} & \multicolumn{1}{c}{A-RRM} \vline & \multicolumn{1}{c|}{AT} & \multicolumn{1}{c}{A-RRM} \vline & \multicolumn{1}{c|}{AT} & \multicolumn{1}{c}{A-RRM} \vline & \multicolumn{1}{c|}{AT} & \multicolumn{1}{c}{A-RRM} \vline & \multicolumn{1}{c|}{AT} & \multicolumn{1}{c}{A-RRM} \\ \hline \hline
            \multirow{5}{*}{\rotatebox{90}{$\epsilon_{test}$}} 
            & 0.00 & \textbf{97} & 96 & 63 & \textbf{95} & 57 & \textbf{97} & 58 & \textbf{96} & 26 & \textbf{86} \\ \cline{2-12}
            & 0.10 & \textbf{95} & 93 & 64 & \textbf{92} & 71 & \textbf{94} & 61 & \textbf{93} & 20 & \textbf{82} \\ \cline{2-12}
            & 0.25 & \textbf{93} & 90 & 83 & \textbf{91} & 88 & \textbf{92} & 84 & \textbf{90} & 74 & \textbf{81} \\ \cline{2-12}
            & 0.50 & \textbf{91} & 88 & \textbf{94} & 91 & \textbf{94} & 90 & \textbf{90} & 88 & \textbf{97} & 80 \\ \cline{2-12}
            & 1.00 & \textbf{86} & 83 & \textbf{95} & 90 & \textbf{94} & 86 & \textbf{88} & 83 & \textbf{98} & 77 \\ \hline
        \end{tabular}
    \end{center}
\end{table*}

\textbf{Implementation Details:}
In this experiment, we divide MNIST-3 into a training set of 18623 examples (5923 0's, 6742 1's, and 5958 2's), and 3147 testing examples. Using Tensorflow 2.10 \citep{tensorflow2015-whitepaper}, 100 epochs of training are performed for the MNIST-3 architecture outlined in Section \ref{sec:architectures} for each of the baseline loss functions (ELR, CCE, MAE, and MSE) and training label contamination levels of 55\%, 60\%, and 65\% for a total 12 baseline experiments. In order the to assess the benefit of our method, the RRM algorithm is executed for each of the baseline loss functions and contamination levels. More specifically, 10 iterations of RRM are executed with $\sigma = 10$ epochs per iteration for a total of 100 epochs. For RRM, the hyperparameter settings of $\mu$ and $\gamma$  are set to 0.5 and 0.4, although auto-tuning could have been implemented as in Section \ref{sec: ExperimentControllingChi}. All methods employ stochastic gradient descent (SGD) with a learning rate ($\eta$) of $0.1$. The test set accuracy of each experiment is shown in Table \ref{tab:rrm-wrapping}, with the RRM experimental results placed in parentheses.

\textbf{Analysis:} For every method tested, the RRM-wrapped approach confers a test set performance benefit over the baseline approach. The additional computation time incurred from the \texttt{Re-weight}($\gamma, \mu$) step is an average of 2.86 seconds per execution of every iteration, for a total of 28.6 additional seconds.



\subsection{A-RRM: Granting ERM Robustness to Adversarial and Noisy Data Regime}\label{sec:advexp}
MNIST-10 is utilized to evaluate our method in a setting of both adversarial feature perturbation and label contamination.

\textbf{Adversarial Training (AT) Benchmark:}
 As benchmark, we execute a standard adversarial training (AT) approach, in which we execute A-RRM, omitting the $\texttt{Re-weight}$ step. Indeed, it is common to perform training on data augmented with adversarial examples crafted through FGSM \citep{MadryMSTV18}. Further, using this benchmark, we may highlight the value of re-weighting.
 
 \textbf{Test-set Perturbation $\epsilon_{test}$}: In the experiments, upon conclusion of the A-RRM training for which a trained $\theta^*$ is obtained, we executed an FGSM attack on each member $(x,y)$ of the test set via $x + \epsilon_{test} \cdot \nabla J(\theta^*; x,y)$.
 
 In this way, in analogous fashion to the possible misalignment between estimate $C'$ and the actual contamination level $C$, we explore the harm to performance that a misalignment between our training perturbation $\epsilon$ and a true test-set perturbation $\epsilon_{test}$ could have. 

\textbf{Implementation Details:}
In this experiment twenty percent of the training data is set aside for validation purposes. Using Tensorflow 2.10 \citep{tensorflow2015-whitepaper}, 50 iterations of A-RRM are executed with $\sigma = 10$ epochs per iteration for a total of 500 epochs for a given hyperparameter setting. For A-RRM, the hyperparameter settings of $\mu$ and $\gamma$ at 0.5 and 2.0, respectively, are based on a search to optimize validation set accuracy. For contrast, we perform a comparable 500 epochs using AT. Both AT and A-RRM employ stochastic gradient descent (SGD) with a learning rate ($\eta$) of $0.1$. An $\epsilon = 1.0$ is used for all training image perturbations.

\textbf{Performance Analysis:} For each of the 0\%, 5\%, 10\%, 20\%, and 30\% training label contamination levels, we compare adversarial training (AT) and A-RRM performance under various regimes of test set perturbation ($\epsilon_{test} \in \{0.0, 0.1, 0.25, 0.5, 1.0\}$). In Table \ref{tab:mnistadv} we show the test set accuracy achieved when validation set accuracy peaks. We can see that training with an $\epsilon = 1.0$ and testing with lower $\epsilon_{test}$ levels of $0.00, 0.10,$ and $0.25$, results in a drastic degradation in accuracy for AT for contamination levels greater than 0\%. This performance collapse is not observed when using A-RRM. Given that the $\epsilon$ employed during training may not match test/production environment ($\epsilon_{test}$), our findings suggest that A-RRM can confer a greater benefit than AT in these scenarios.

\textbf{Convergence of $u$}: We examine the $u_i$-value associated with each training observation $i$ from iteration-to-iteration of Algorithm \ref{alg: RRM}. Table \ref{tab:mnistadv_uvalues} summarizes the progression of the $u_i$-vector across its 49 updates for the dataset contamination level of 20\%. Column “1. iteration” shows the distribution of $u_i$-values following the first u-optimization for both the 9600 contaminated training observations and the 38400 clean training observations. Initially, all $u_i$-values are approximately equal to 0.0. It is once again observed that, over the course of iterations, the $u_i$-values noticeably change. In column “10. iteration” it can be seen that a significant number of the $u_i$-values of the contaminated training observations achieve negative values, while a large majority of the $u_i$-values for the clean training observations remain close to 0.0. Finally, column “49. iteration” displays the final $u_i$-values. 9286 out of 9600 of the contaminated training observations have achieved a $u_i \in (-2.08, -1.56] \cdot 10^{-5}$. This means these training observations are removed, or nearly-so, from consideration because this value cancels the nominal probability $1/N$ = 2.08 $\cdot$ $10^{-5}$. It is observed that a large majority (35246/38400) clean training observations remain with their nominal probability. This helps explain the performance benefit of A-RRM over AT. A-RRM "removes" the contaminated data points in-situ, whereas AT does not. It appears that under adversarial training regimes with contaminated training data, it is essential to identify and "remove" the contaminated examples, especially if the level adversarial perturbation encountered in the test set is unknown, or possibly lower than the level of adversarial perturbation applied to the training set.

\begin{table}[h!]
    \small
    \begin{center}
        \caption{{\bf \underline{MNIST-10}} Evolution of $u$ across $|C| = 9600$ contaminated data points and $N - |C| = 38400$ clean data points. Note that $ 1/N = 2.08 \cdot 10^{-5}$.}
        \label{tab:mnistadv_uvalues}
        \begin{adjustbox}{width=\columnwidth,center}
        \begin{tabular}{c|c||c|c||c|c||c|c}
            \multicolumn{2}{c||}{} & \multicolumn{2}{c||}{1. iteration} & \multicolumn{2}{c||}{10. iteration} & \multicolumn{2}{c}{49. iteration} \\ \hline
            &  & $C$ & $[N] \setminus C$ & $C$ & $[N] \setminus C$ & $C$ & $[N] \setminus C$ \\ \hline \hline
            \multirow{6}{*}{\rotatebox{90}{$u_i$ value ($10^{-5}$)}} 
            & $\gg$ 0          & 0    & 1     & 0    & 4    & 0    & 25    \\ \cline{2-8}
            & $\approx$ 0      & 8844 & 38385 & 2058 & 37524 & 91   & 35246 \\ \cline{2-8}
            & (-0.52,  0.00)   & 0    & 0     & 7    & 36    & 146  & 1655  \\ \cline{2-8}
            & (-1.04, -0.52]   & 0    & 0     & 41   & 45    & 43   & 155   \\ \cline{2-8}
            & (-1.56, -1.04]   & 756  & 14    & 415  & 174   & 34   & 168   \\ \cline{2-8}
            & (-2.08, -1.56]   & 0    & 0     & 7079 & 617   & 9286 & 1151  \\ \hline
        \end{tabular}
        \end{adjustbox}
    \end{center}
\end{table}

\section{Conclusion}\label{sec:conclusions}

A-RRM/RRM algorithm exhibits robustness in a variety of data domains, data contamination schemes, model architectures and machine learning applications. In truth, we have presented a meta-algorithm that operates on loss function methodologies, and further has the capacity to incorporate benchmarks/estimates of label contamination and adversarial feature perturbation. Furthermore, the method is virtually hyperparameter-less, upon utilizing an auto-tuning feature that enables control over the data-pruning action of \texttt{Re-weight}, to be used with an estimate of the label contamination level.

Our experiments indicate RRM not only amplifies test accuracy of methods that  don't account for contamination, but even those already robust to contamination. In the MNIST-10 example we show that conducting training in preparation for deployment environments with varied levels of adversarial attacks, one can benefit from implementation of the A-RRM algorithm. This can lead to a model more robust across levels of both feature perturbation and high levels of label contamination. We also demonstrate the mechanism by which A-RRM operates and confers superior results: by automatically identifying and removing the contaminated training observations at training time execution.

\appendix

\section*{Impact Statement}
This work seeks to advance the field of Machine Learning, particularly in the area of model robustness. While various secondary implications may emerge from its application, we do not identify any immediate or critical concerns that warrant specific emphasis in this context.


%
\nocite{langley00}

\bibliography{ArXiv}
\bibliographystyle{icml2025}

\newpage
\appendix
\onecolumn
\section{Appendix/supplemental material}

\subsection{Section \ref{methodology} Proofs}
\SolvingForU*
\begin{proof}
For any set $C$, let $\iota_C(x) = 0$ and $\iota_C(x) = \infty$ otherwise. We recognize that $u^\star$ is a solution of the minimization problem if and only if it is a minimizer of the function $h$ given by
\[
h(u) = \sum_{i = 1}^N  \Big(c_i/N  + u_i c_i + \frac{\gamma}{2} |u_i| + \iota_{[0,\infty)}(1/N + u_i)\Big)  + \iota_{\{0\}}\Big(\sum_{i=1}^N u_i \Big)
\]
Thus, because $h(u) > -\infty$ for all $u\in\mathbb{R}^N$ and $h$ is convex, $u^\star$ is a solution of the minimization problem if and only if $0 \in \partial h(u^\star)$. We proceed by characterizing $\partial h$. 

Consider the univariate function $h_i$ given by 
\[
h_i(u_i) = c_i/N  + u_i c_i + \frac{\gamma}{2} |u_i| + \iota_{[0,\infty)}(1/N + u_i).
\]
For $u_i \geq -1/N$, the Moreau-Rockafellar sum rule reveals that 
\[
\partial h_i(u_i) = c_i + \begin{cases}
    \{\frac{\gamma}{2}\} & \mbox{ if } u_i > 0\\
    [-\frac{\gamma}{2}, \frac{\gamma}{2}] & \mbox{ if } u_i = 0\\
    \{-\frac{\gamma}{2}\} & \mbox{ if } -1/N < u_i < 0\\
    (-\infty, -\frac{\gamma}{2}] & \mbox{ if } u_i = -1/N.
\end{cases}
\]
For $u = (u_1, \dots, u_N) \in [-1/N, \infty)^N$, we obtain by Proposition 4.63 in \cite{primer} that 
\[
\partial\Big(\sum_{i=1}^N h_i\Big)(u) = \partial h_1(u_1) \times \cdots \times \partial h_N(u_N). 
\]
Let $h_0$ be the function given by $h_0(u) = \iota_{\{0\}}(\sum_{i=1}^N u_i)$. 
Again invoking the Moreau-Rockafellar sum rule while recognizing that the interior of the domain of $\sum_{i=1}^N h_i$ intersects with the domain of $h_0$, we obtain 
\[
\partial h(u) = \partial\Big(\sum_{i=1}^N h_i\Big)(u) + \partial h_0(u) = \partial h_1(u_1) \times \cdots \times \partial h_N(u_N) + \begin{bmatrix}
    1\\
    \vdots\\
    1
\end{bmatrix}\mathbb{R}
\]
for any $u = (u_1, \dots, u_N)$ with $u_i \geq -1/N$, $i = 1, \dots, N$, and $\sum_{i=1}^N u_i = 0$. Hence, $u^* \in U$ is optimal if and only if for some $\lambda\in \mathbb{R}$,
\begin{equation} \label{eq:: LambdaKKTConditions}
\lambda \in \begin{cases}
    \{c_i + \frac{\gamma}{2}\} & \mbox{ if } u_i^\star > 0\\
    [c_i -\frac{\gamma}{2}, c_i + \frac{\gamma}{2}] & \mbox{ if } u_i^\star = 0\\
    \{c_i -\frac{\gamma}{2}\} & \mbox{ if } u_i^\star \in (-1/N, 0)\\
    (-\infty, c_i -\frac{\gamma}{2}] & \mbox{ if } u_i^\star = -1/N. 
    \end{cases}
\end{equation}

It follows that any such $\lambda$ must then satisfy $\lambda \leq c_{\min} + \frac{\gamma}{2}$, as the above list of cases indicate for arbitrary $i$ and no matter the value of $u_i^*$, for any optimal $u^*\in U$. 

We proceed to show that in fact, $u^* \in U$ solves \eqref{eq:: inner - U} if and only if $\lambda = c_{min} + \frac{\gamma}{2}$ satisfies the conditions of \eqref{eq:: LambdaKKTConditions} with respect to $u^*$ for all $i$. Suppose $u^*\in U$ is optimal and that there exists a corresponding $\lambda$ satisfying \eqref{eq:: LambdaKKTConditions} for all $i$, and that it satisfies $\lambda < c_{min} + \frac{\gamma}{2}.$ Then for any optimal $u^* \in U$, there exists no $i$ for which $u_i^* > 0,$ otherwise $c_i = \lambda - \frac{\gamma}{2} < c_{min},$ an impossibility. This means any optimal $u^* \in U$ must satisfy $u^* \leq 0,$ for which then we would conclude that in fact $u^* = 0,$ and is the unique optimal solution, which in turn yields that $\lambda - \frac{\gamma}{2} \leq c_i \leq \lambda + \frac{\gamma}{2} < c_{min} + \gamma$ for all $i.$ This last inequality reveals that $\lambda^*:= c_{min} + \gamma$ is compatible with the unique optimal solution $u^* = 0 \in U$ in such a case. In summary, if a $\lambda < c_{min} + \frac{\gamma}{2}$ is compatible with all optimal solutions $u^* \in U,$ then $\lambda^*:= c_{min} + \frac{\gamma}{2}$ is as well. 

It follows that $u^* \in U$ solves \eqref{eq:: inner - U} if and only if $\lambda = c_{min} + \frac{\gamma}{2}$ satisfies the conditions of \eqref{eq:: LambdaKKTConditions} with respect to $u_i^*$ for all $i$; hence, the result follows. 
\end{proof}

\QuickSolveU*
\begin{proof}
    Clearly, $u^* \in U$. Let $c_i = J(\theta^*; x_i, y_i)$ for $i = 1,\ldots, N.$ In the proof of Theorem \ref{theorem:: SovlingForU}, it was shown that $u \in U$ solves \eqref{eq:: inner - U} if and only if $\lambda = c_{min} + \frac{\gamma}{2}$ satisfies the conditions of \eqref{eq:: LambdaKKTConditions} with respect to $u_i*$ for all $i$. Therefore, with $S_1 = S_2 = \emptyset,$ Theorem \ref{theorem:: SovlingForU} indicates that $u^* \in U^*_{S_1, S)_2}$ and hence must solve \eqref{eq:: inner - U}.
\end{proof}

\Optimistic*
\begin{proof}
Fix $\theta.$ Then for any $z = (x,y) \in \mathcal{Z}$, the function $\ell( \cdot, z, \theta)$ is linear, and hence Lipschitz with constant $\ell( 1, z, \theta) = J(\theta; x,y) \leq \max_{(x,y) \in \mathcal{Z}} |J(\theta; x,y)| < \infty$.   

By Lemma 3.1 of \cite{distributionally2} and/or Corollary 2 of \cite{Gao2023},
    \begin{align*}
    v_N^{MIX}(\theta):= \min_{\tilde{w}^1, \ldots, \tilde{w}^N \geq 0} ~ & \frac{1}{N}\sum_{i = 1}^N  \ell(\tilde{w}^i, z^i; \theta)\\
    & \text{s.t.} ~~ \frac{1}{N} \sum_{i=1}^N |\tilde{w}^i - w^i| \leq \epsilon
    \end{align*}
    provides the stated approximation of $v(\theta).$

    Upon introducing the change of variable $u_i = \frac{\tilde{w}^i}{N} - \frac{1}{N},$ and applying a Lagrange multiplier $\gamma_\theta$ to the $\epsilon-$ budget constraint (any convex dual optimal multiplier), we recover
\begin{align*}
\min_{u_1, \ldots, u_N} ~ & \sum_{i = 1}^N  \ell(u_i + \frac{1}{N}, z^i; \theta) + \gamma_\theta \sum_{i=1}^N |u_i|\\
& \text{s.t.}  ~~u_i + \frac{1}{N} \geq 0 \;\; i = 1, \ldots, N
\end{align*}
\end{proof}


\newpage

\section{Non-uniform label Contamination Experiment}\label{sec:nonuniform}
Although we explicitly state the use of uniform label noise in Section \ref{sec:ncar}, which is indeed a very common scheme in the literature, our analysis in fact did not rely on this assumption. Towards providing insight into the non-uniform case, we have repeated the experiments of Section \ref{sec:advexp} that produced Table \ref{tab:mnistadv}, but now with non-uniform label noise. More precisely, after uniformly at random selecting $C$
 percent of the training pairs, we proceed to contaminate the label $y_i$ 
 in each pair $(x_i,y_i)$
 in the following non-uniform manner, as outlined below in the transition kernel matrix of (True Label, Contaminated Label) entries. For example, if the true label $y_i = 5$, then instead of uniformly at random drawing an alternative digit $\tilde{y}_i$
 from among $\{0,1,\ldots, 9\}\setminus \{5\}$ we have 
 $\tilde{y}_i = 
 \begin{cases}
     0 &w.p. 0.051\\
    1 & w.p. 0.017\\
    2 & w.p. 0.\\
    3 & w.p. 0.627
 \end{cases} $

\begin{table}[ht!]
    \small
    \begin{center}
        \caption{Contamination Kernel}
        \label{tab:true_vs_contaminated}
        \begin{tabular}{c|c||c|c|c|c|c|c|c|c|c|c}
            \multicolumn{2}{c}{} & \multicolumn{9}{c}{Contaminated} \\ \cline{2-12}
            &  & 0 & 1 & 2 & 3 & 4 & 5 & 6 & 7 & 8 & 9 \\ \hline \hline
            \multirow{10}{*}{\rotatebox{90}{Original}} 
            & 0 & 0 & 0.0769 & 0.0769 & 0.1538 & 0 & 0.0769 & 0.3846 & 0 & 0.1538 & 0.0769 \\ \cline{2-12}
            & 1 & 0 & 0 & 0.3333 & 0.1111 & 0 & 0.1111 & 0.1111 & 0 & 0.3333 & 0 \\ \cline{2-12}
            & 2 & 0.0968 & 0.0645 & 0 & 0.2581 & 0.0323 & 0 & 0.0968 & 0.1935 & 0.2581 & 0 \\ \cline{2-12}
            & 3 & 0 & 0 & 0.1250 & 0 & 0 & 0.1250 & 0 & 0.1250 & 0.6250 & 0 \\ \cline{2-12}
            & 4 & 0.1111 & 0.0370 & 0.0741 & 0.0741 & 0 & 0.0741 & 0.2222 & 0.0370 & 0.1111 & 0.2593 \\ \cline{2-12}
            & 5 & 0.0508 & 0.0169 & 0 & 0.6271 & 0.0169 & 0 & 0.1525 & 0 & 0.1017 & 0.0339 \\ \cline{2-12}
            & 6 & 0.2353 & 0.1765 & 0.0588 & 0.0588 & 0.0588 & 0.1765 & 0 & 0 & 0.2353 & 0 \\ \cline{2-12}
            & 7 & 0.0500 & 0.2250 & 0.2000 & 0.1250 & 0 & 0 & 0 & 0 & 0.2000 & 0.2000 \\ \cline{2-12}
            & 8 & 0.1071 & 0.0357 & 0.1071 & 0.3571 & 0.1071 & 0.0714 & 0.0714 & 0.1071 & 0 & 0.0357 \\ \cline{2-12}
            & 9 & 0.0638 & 0.1702 & 0 & 0.2128 & 0.1702 & 0.1702 & 0.0213 & 0.0851 & 0.1064 & 0 \\ \hline
        \end{tabular}
    \end{center}
\end{table}

These entries were generated by the confusion matrix of an imperfect MNIST classifier. The results from this new experiment confirm the performance benefits that were observed (compare to Table \ref{tab:mnistadv} under conditions of uniform label contamination.

\begin{table*}[ht!]
    \small
    \begin{center}
        \caption{{\bf \underline{MNIST-10}} Test accuracy (\%) for AT and A-RRM under different levels of contamination $C$, training perturbation $\epsilon = 1.0$, and test-set adversarial perturbation $\epsilon_{test}$. }
        \label{tab:mnistadv2}
        \begin{tabular}{c|c||c|c|c|c|c|c|c|c|c|c}
            \multicolumn{2}{c}{} & \multicolumn{10}{c}{$C$} \\ \hline
            &  & \multicolumn{2}{c}{0\%} \vline & \multicolumn{2}{c}{5\%} \vline & \multicolumn{2}{c}{10\%} \vline & \multicolumn{2}{c}{20\%} \vline & \multicolumn{2}{c}{30\%} \\ \cline{3-12}
            &  & \multicolumn{1}{c|}{AT} & \multicolumn{1}{c}{A-RRM} \vline & \multicolumn{1}{c|}{AT} & \multicolumn{1}{c}{A-RRM} \vline & \multicolumn{1}{c|}{AT} & \multicolumn{1}{c}{A-RRM} \vline & \multicolumn{1}{c|}{AT} & \multicolumn{1}{c}{A-RRM} \vline & \multicolumn{1}{c|}{AT} & \multicolumn{1}{c}{A-RRM} \\ \hline \hline
            \multirow{5}{*}{\rotatebox{90}{$\epsilon_{test}$}} 
            & 0.00 & 96.5 & \textbf{97.3} & 93.6 & \textbf{95.6} & 60.4 & \textbf{87.8} & 32.2 & \textbf{92.4} & 58.3 & \textbf{89.2} \\ \cline{2-12}
            & 0.10 & 93.4 & \textbf{95.2} & 89.3 & \textbf{92.4} & 63.2 & \textbf{84.7} & 42.5 & \textbf{89.5} & 56.1 & \textbf{81.7} \\ \cline{2-12}
            & 0.25 & 92.4 & \textbf{93.1} & 87.9 & \textbf{90.6} & \textbf{86.9} & 86.3 & 80.6 & \textbf{89.3} & 69.0 & \textbf{79.8} \\ \cline{2-12}
            & 0.50 & \textbf{92.0} & 90.9 & \textbf{90.3} & 89.8 & \textbf{94.4} & 91.2 & \textbf{92.9} & 89.2 & \textbf{85.2} & 81.6 \\ \cline{2-12}
            & 1.00 & \textbf{89.4} & 85.5 & \textbf{90.3} & 86.8 & \textbf{94.9} & 92.6 & \textbf{93.9} & 86.6 & \textbf{81.8} & 77.8 \\ \hline
        \end{tabular}
    \end{center}
\end{table*}

\section{Additional Data Experiments} \label{sec:AdditionalExperiments}

\subsection{Toxic Comments} \label{sec:toxicresults}

{\bf Dataset:}
Toxic Comments\footnote[3]{https://kaggle.com/competitions/jigsaw-toxic-comment-classification-challenge} is a multi-label classification problem from JIGSAW that consists of Wikipedia comments labeled by humans for toxic behavior.  Comments can be any number (including zero) of six categories: toxic, severe toxic, obscene, threat, insult, and identity hate.  We convert this into a binary classification problem by treating the label as either none of the six categories or at least one of the six categories.  This dataset is a public dataset used as part of the Kaggle Toxic Comment Classification Challenge.  

{\bf Architecture:}  We use a simple model with only a single convolutional layer.  A pretrained embedding from FastText is first used to map the comments into a 300 dimension embedding space, followed by a single convolutional layer with a kernel size of two with a ReLU activation layer followed by a max-pooling layer.  We then apply a 36-unit dense layer, followed by a 6 unit dense layer with sigmoid activation.  Binary cross-entropy is used for the loss function.

{\bf Experiment: }
We use the Toxic Comments dataset to test the efficacy of RRM on low prevalence text data.  The positive (toxic) comments consist of only 3\% of the data and we contaminate anywhere from 1\% to 20\% of the labels.  There are a total of 148,000 samples, and we set aside 80\% for training and 20\% for test.  $\sigma = 2$ with 3 iterations of the heuristic algorithm results in a total of 6 epochs, and ERM is run for a total of 6 epochs to make the results comparable.  Since the data is highly imbalanced, we look at the area under the curve of the precision/recall curve to assess the performance of the models.  Unsurprisingly, as the noise increase, the model performance decreases.  We note that RRM outperforms ERM across all noise levels tested, though as the noise increase, the gap between RRM and ERM decreases.


\begin{table}[ht!]
   \small
   \centering
   \caption{{\bf \underline{Toxic Comments}} Comparison of training and test area under the precision/recall curve for ERM and RRM at noise levels ranging from 1\% to 20\%.}
   \label{tab:toxic_prAUC}
   \begin{adjustbox}{center}
   \begin{tabular}{c||c|c|c|c|c|c}
       \hline
       \multicolumn{1}{c||}{Method} & \multicolumn{6}{c}{Percentage Contaminated Training Data} \\ \cline{2-7} 
       & 1\% & 5\% & 7\% & 10\% & 15\% & 20\% \\ \hline
       ERM (train)& 0.2904 & 0.2006 & 0.1589 & 0.1302 & 0.1073 & 0.0920 \\ \hline
       RRM (train)& 0.6875 & 0.4458 & 0.3805 & 0.3087 & 0.2438 & 0.1966 \\ \hline \hline
       ERM (test)& 0.5861 & 0.3970 & 0.3246 & 0.2550 & 0.2013 & 0.1717 \\ \hline
       RRM (test)& \textbf{0.6705} & \textbf{0.4338} & \textbf{0.3619} & \textbf{0.2824} & \textbf{0.2208} & \textbf{0.1861} \\ \hline
   \end{tabular}
   \end{adjustbox}
\end{table}

{\bf Takeaway:} 
The Toxic Comment example presents another challenging classification problem, characterized by a low prevalence target class amidst label noise. Our experiments demonstrate that as the amount of label noise increases, standard methods become increasingly ineffective. However, RRM remains reasonably robust under varying degrees of label contamination. Therefore, RRM could be a valuable addition to the set of tools being developed to enhance the robustness of AI-based decision engines.

\subsection{IMDb}\label{sec:imdbresults}
{\bf Dataset \citep{IMDb}: } A binary classification dataset consisting of 50000 movie reviews each assigned a positive or negative sentiment label. 25000 reviews are selected randomly for training and the remaining are used for testing. 25\%, 30\%, 40\%, and 45\% of the labels of the training reviews are randomly selected and swapped from positive sentiment to negative sentiment, and vice versa, to achieve four training datasets of desired levels of label contamination. 

{\bf Architecture:} Transformer architectures have achieved SOA performance on the IMDb dataset sentiment analysis task \citep{bert, uda}. As such, we a adopt a reasonable transformer architecture to assess RRM. We utilize the DistilBERT \citep{distilbert} architecture with low-rank adaptation (LoRA) \citep{lora} for large language models, which reduces the number of trainable weights from 67584004 to 628994. In this manner, we reduce the computational burden, while maintaining excellent sentiment analysis performance. Binary cross-entropy is employed for the loss function.

{\bf Experiment:} Twenty percent of the training data is set aside for validation purposes. Using Pytorch 2.1.0, 30 iterations of RRM are executed, with $\sigma = 10$ epochs per iteration for a total of 300 epochs for a given hyperparameter setting. For RRM, the hyperparameter settings of $\mu$ and $\gamma$ at 0.5 and 0.4, respectively, are based on a search to optimize validation set accuracy. For contrast, we perform a comparable 300 epochs using ERM. Both ERM and RRM employ stochastic gradient descent (SGD) with a learning rate ($\eta$) of $0.001$. In Table \ref{tab:imdbresults} we record both the test set accuracy achieved when validation set accuracy peaks, as well as the maximum test set accuracy. At these high levels of contamination RRM consistently achieves a better maximum test set accuracy.

\begin{table}[ht!]
    \small
    \centering
    \caption{ {\bf \underline{IMDb}} Test accuracy (\%) for ERM and RRM under different levels of contamination. Test set accuracy at peak validation accuracy and maximum test set accuracy are recorded.}
    \label{tab:imdbresults}
    \begin{tabular}{c||c|c|c|c}
        \hline
        \multicolumn{1}{c||}{Method} & \multicolumn{4}{c}{Percentage Contaminated Training Data} \\ \cline{2-5} 
        & 25\% & 30\% & 40\% & 45\% \\ \hline
        ERM & \textit{90.2}, 90.2 & 89.5, 89.6 & 86.4, 86.6 & \textit{80.7}, 81.1 \\ \hline
        RRM & 90.1, \textbf{90.4} & \textit{90.2}, \textbf{90.4} & \textit{88.4}, \textbf{88.7} & 76.9, \textbf{82.6} \\ \hline
    \end{tabular}
\end{table}

{\bf Takeaway:} We demonstrate that RRM can confer benefits to the sentiment analysis classification task using pre-trained large models under conditions of high label contamination. The success of fine-tuning in LLMs depends, in large part, on access to high quality training examples. We have shown that RRM can mitigate this need by allowing effective training in scenarios of high training data contamination. As such, resource allocation dedicated to dataset curation may be lessened by the usage of RRM.

\subsection{Tissue Necrosis}\label{sec:necrosis}

\textbf{Dataset}: A binary classification dataset consisting of 7874 256x256-pixel hematoxylin and eosin (H\&E) stained RGB images derived from \citep{necrosis}. The training dataset consists of 3156 images labeled non-necrotic, as well as 3156 images labeled necrotic. The training images labeled non-necrotic contain no necrosis. However, only 25\% of the images labeled necrotic contain necrotic tissue. This type of label error can be expected in cases of weakly-labeled Whole Slide Imagery (WSI). Here, an expert pathologist will provide a slide-level label for a potentially massive slide consisting of gigapixels, but they lack time or resources to provide granular, segmentation-level annotations of the location of the pathology in question. Also, the diseased tissue often occupies a small portion of the WSI, with the remainder consisting of normal tissue. When the gigapixel-sized WSI is subsequently divided into sub-images of manageable size for typical machine-learning workflows, many of the sub-images will contain no disease, but will be assigned the "weak" label chosen by the expert for the WSI. The test dataset consists of 718 necrosis and 781 non-necrosis 256x256-pixel H\&E images, which were also derived from \citep{necrosis}. For both the training and test images, \citep{necrosis} provide segmentation-level necrosis annotations, so we are able to ensure a pristine test set, and, in the case of the training set, we were able to identify the contaminated images for the purpose of algorithm evaluation.

\textbf{Architecture}: Consistent with the computational histopathology literature \citep{dnn_he_survey}, we employ a convolutional neural network (CNN) architecture for this classification task. In particular, a ResNet-50 architecture with pre-trained ImageNet weights is harnessed. The classification head is removed and replaced with a dense layer of 512 units and ReLU activation function, followed by an output layer with a single unit using a sigmoid activation function. All weights, with the exception of the new classification head are frozen, resulting in 1050114 trainable parameters out of 24637826. Binary cross-entropy is employed for the loss function.

{\bf Experiment:}Twenty percent of the training data is set aside for validation purposes, including hyperparameter selection. 60 iterations of RRM are executed, with $\sigma = 10$ epochs per iteration, for a total of 600 epochs for a given hyperparameter setting. For RRM, the hyperparameter settings of $\mu$ and $\gamma$ at 0.5 and 0.016, respectively, are based on a search to optimize validation set accuracy. For contrast, we perform a comparable 600 epochs using ERM. Both ERM and RRM employ stochastic gradient descent (SGD) with a learning rate ($\eta$) of $5.0$ and $1.0$, respectively. RRM achieves a test set accuracy at peak validation accuracy of \textbf{74.6}, and a maximum test set accuracy \textbf{77.2}, whereas ERM achieves 71.7 and 73.2, respectively. RRM appears to confer a performance benefit under this regime of weakly labeled data. 

{\bf Takeaway:} In the Tissue Necrosis example, we demonstrate that RRM also confers accuracy benefits to the necrosis identification task provided weakly labeled WSIs. Again, RRM can mitigate the need for expert-curated, detailed pathology annotations, which are costly and time-consuming to generate.

\end{document}